\documentclass[%
 aip,
 amsmath,amssymb,
preprint,%
author-year,%
]{revtex4-1}

\usepackage{graphicx}
\usepackage{dcolumn}
\usepackage{bm}
\usepackage[mathlines]{lineno}

\usepackage[utf8]{inputenc}
\usepackage[T1]{fontenc}
\usepackage{mathptmx}
\usepackage{etoolbox}
\usepackage{algorithm2e}
\usepackage{amsthm}
\usepackage{hyperref}
\usepackage{cleveref}
\hypersetup{
    colorlinks=true,
    linkcolor=cyan,
    filecolor=magenta,      
    urlcolor=cyan,
    citecolor=cyan,
    }
\usepackage{caption}
\usepackage{subcaption}
\usepackage{booktabs}
\usepackage{multirow}
\newtheorem{theorem}{Theorem}
\newtheorem{lemma}{Lemma}

\newtheorem{definition}{Definition}

\makeatletter
\def\@email#1#2{%
 \endgroup
 \patchcmd{\titleblock@produce}
  {\frontmatter@RRAPformat}
  {\frontmatter@RRAPformat{\produce@RRAP{*#1\href{mailto:#2}{#2}}}\frontmatter@RRAPformat}
  {}{}
}%
\makeatother
\begin{document}
\RestyleAlgo{ruled} 

\title[Random Walk in Random Permutation Set Theory]{Random Walk in Random Permutation Set Theory}
\author{Jiefeng Zhou}
    \altaffiliation[Also at ]{Yingcai Honors College, University of Electronic Science and Technology of China, Chengdu, 610054, China}
\affiliation{Institute of Fundamental and Frontier Science, University of Electronic Science and Technology of China, Chengdu, 610054, China}

\author{Zhen Li}
\affiliation{China Mobile Information Technology Center, Beijing, 100029, China}

\author{Yong Deng}
\email[\textbf{Corresponding author: }]{dengentropy@uestc.edu.cn, prof.deng@hotmail.com}
    \altaffiliation[Also at ]{School of Medicine, Vanderbilt University, Nashville, Tennessee, 37240, USA}
\affiliation{Institute of Fundamental and Frontier Science, University of Electronic Science and Technology of China, Chengdu, 610054, China}


\date{\today}

\begin{abstract}
Random walk is an explainable approach for modeling natural processes at the molecular level. The Random Permutation Set Theory (RPST) serves as a framework for uncertainty reasoning, extending the applicability of Dempster-Shafer Theory. Recent explorations indicate a promising link between RPST and random walk. In this study, we conduct an analysis and construct a random walk model based on the properties of RPST, with Monte Carlo simulations of such random walk. Our findings reveal that the random walk generated through RPST exhibits characteristics similar to those of a Gaussian random walk and can be transformed into a Wiener process through a specific limiting scaling procedure. This investigation establishes a novel connection between RPST and random walk theory, thereby not only expanding the applicability of RPST, but also demonstrating the potential for combining the strengths of both approaches to improve problem-solving abilities.
\end{abstract}


\maketitle

\section{Introduction}\label{introduction}

Random walk models have been used to simulate various natural processes. These models are particularly useful for understanding molecular-level dynamics \citep{kessing2022longrange, ansari2012monte, thompson2022enhanced}, complex networks \citep{zhang2018measure, wang2021convex}, and so on. Correlated walks are specialized category of random walks. In correlated walks, the moving particles have a memory of their previous steps. This memory affects the direction of the next step, making the order of the steps important \citep{tojo1996correlated}.

The Dempster-Shafer evidence theory (DSET), also known as evidence theory, is a framework used for reasoning under uncertainty \citep{ap1967upper, shafer1976mathematical}. In contrast to probability theory, DSET utilizes mass functions to assign beliefs to subsets of a power set, rather than to individual outcomes in the sample space, allowing for a more flexible approach to defining belief assignments. Thus, DSET has been applied to various fields \citep{liu2023new, gao2023complex, liu2024learnable, yang2024novel, contreras2023belief, cui2022belief} However, DSET struggles to handle ordered information in certain real-world problems. To address this limitation, the random permutation set theory(RPST) is introduced \citep{deng2022random}. By introducing the concept of permutation events, RPST effectively considers the order of elements and expands the power set and mass function into the permutation events space (PES) and permutation mass function (PMF). Similar to Shannon entropy \citep{shannon1948mathematical} in probability and Deng entropy \citep{deng2016deng} in DSET, the RPS entropy is proposed by \cite{chen2023entropy} to quantify the uncertainty in RPST. 

One research area of focus in DSET and RPST is its physical implications. For instance, in DSET \cite{li2023normal} derived normal distribution from maximum Deng entropy \cite{kang2019maximum}, and \cite{zhao2024linearity} found intriguing linearity in Deng entropy. In RPST,  \cite{zhan2024random} expanded order information to encompass more complex relationships. \cite{zhou2022marginalization} used cooperative game to interpret RPST. \cite{deng2022maximum} discovered the PMF condition for maximum entropy in RPST. And \cite{zhao2023information} further demonstrated that the information dimension associated with this PMF condition is $2$, similar to the fractal dimension of Brownian motion. Moreover, the mean square distance, denoted as $\bar{r^{2}}$, in a Brownian motion is proportional to the time elapsed. In our previous research, it is showed that the limit form of maximum RPS entropy is $e \cdot (n!)^{2}$, exhibiting similarities to $\bar{r^{2}}$ \citep{zhou2024limit}. These collective discoveries hint at a potential connection between RPST and Brownian motion, or random walk in mathematics. 

In this paper, we conduct an in-depth analysis of RPST and construct random variables based on its properties. We then generate a random walk using these random variables. Finally, we demonstrate that this type of random walk shares similarities with a Gaussian random walk and can be converted into a Wiener process through scaling.

In general, this paper successfully build a bridge between RPST and random walk theory. Such correlation sheds light on the significance of RPST and its potential applications within random walk framework. This connection not only broadens the utility of RPST but also indicating the potential of combining strengths of both methodologies to enhance problem-solving capabilities.

The following parts of this article is structured as follows. \cref{sec:preliminary} introduces some key concepts related to this work. In \cref{sec:explore random walk}, the construction of RPST-generated random walk is presented. Finally, this article is summarized in \cref{sec:conclusion}. The proof of this work is attached in \cref{sec: proof}.

\section{Preliminaries}\label{sec:preliminary}
Some key concepts about this work are introduced in this section.

\subsection{Sample space and mass function}
\begin{definition}
\textnormal{(Sample space).} A sample space $\Omega$ is a mathematical set that contains all possible base events $E_i$, the cardinality of sample space is denoted as $|\cdot |$. the \textit{power set} of $\Omega$ is marked as $2^\Omega$.
\begin{align}
    \Omega = \left\{ E_1, E_2, E_3, \dots, E_n  \right\},\quad |\Omega|=n.
\end{align}
\end{definition}

\begin{definition}
\textnormal{(Mass function).} A mass function $\mathcal{M}(\cdot)$ a function that assigns a belief to each subset of a sample space $\Omega$, $\mathcal{M}: 2^\Omega \to [0,1]$, with the following constraints.
\begin{align}
    \sum_{i} \mathcal{M}(i) = 1, \quad \mathcal{M}(i) \geq 0,\quad \mathcal{M}(\emptyset) = 0.
\end{align}
\end{definition}

\subsection{Random permutation set theory}
By introducing ordered information, the random permutation set theory (RPST) successfully extends the scope of evidence theory. Some fundamental concepts of RPST are introduced below.

\subsubsection{Random permutation set}

\begin{definition}
\textnormal{(Permutation event space, PES).} The permutation event space (PES) is a set that contains all possible permutations of base events of $\Omega$. 

\begin{align}
    PES(\Omega) &= \left\{p_{ij} | i=0,1,\dots, n; j=1,2,\dots, P(n,i)   \right\} \nonumber \\
    & = \left\{ \emptyset, [E_1], [E_2], \dots, [E_n], [E_1, E_2], [E_2, E_1], \dots \nonumber \right. \\
    &\left. [E_{n-1}, E_{n}], [E_{n}, E_{n-1}], \dots,\nonumber \right. \\ 
&\left. [E_1, E_2, \dots, E_n], [E_n, E_{n-1}, \dots, E_1] \right\},
\end{align}

\end{definition}
where $P(n,i) = n!/(n-i)!$ is the $i$-permutation of $n$.

\begin{definition}
    \textnormal{(Permutation mass function, PMF).} A permutation mass function (PMF) $M$ is a mapping $M: PES(\Omega) \to [0,1]$, with constraints
\begin{align}
M(\emptyset)=0, \quad, \sum_{p \in PES(\Omega)} M(p) = 1.
\end{align}
\end{definition}

The \textbf{random permutation set (RPS)} consists of a permutation event from $PES(\Omega)$ and its associated permutation mass function (PMF) $M$: $RPS(\Omega) = \left\{A, M(A)| A \in PES(\Omega)\right\}$.

\subsubsection{RPS entropy}

Similar to entropy methods as uncertainty measure in evidence theory, RPS entropy has been proposed recently \citep{chen2023entropy}. What is more, the maximum RPS entropy and its limit form are also introduced and proved \citep{deng2022maximum, zhou2024limit}.

\begin{definition}
\textnormal{(RPS entropy).} The RPS entropy of a RPS $RPS(\Omega) = \left\{A, M(A)\right\}$ is defined as 

\begin{align}
H_RPS(M) = -\sum_{A\in PES(\Omega)}M(A) \log \left( M(A)/(F(|A|)-1  \right),
\end{align}
where $|A|$ is the cardinality of permutation event $A$, and $F(i)= \sum_{j=0}^{i}P(i,j)$.
\end{definition}
RPS entropy is fully compatible with Deng entropy \citep{deng2016deng} as used in evidence theory, and Shannon entropy \citep{shannon1948mathematical} in probability theory. Such uncertainty measures have been provided insights for other uncertainty measures like distance \citep{chen2023distance}, divergence \citep{chen2023permutation, zeng2023new}, information measures \citep{zhao2023information, kharazmi2023deng, ortiz2024reformulation} and so on.

\cite{deng2022maximum} delved and proved the following PMF condition of maximum RPS entropy:
\begin{align}
M(A) = \frac{F(|A|)-1}{\sum_{i=1}^{n}[P(n,i)(F(i)-1)]}.
\end{align}

The corresponding maximum RPS entropy for such PMF condition is then expressed as:
\begin{align}
H_{max-RPS} = \log \left(\sum_{i=1}^{n}[P(n,i)(F(i)-1)] \right).
\end{align}
 
 And the limit form of maximum RPS entropy can be simplified as \citep{zhou2024limit}:

\begin{align}
H_{max-RPS} \approx e \cdot (n!)^2.
\end{align}

This elegant result offers valuable insights into the physical significance of RPST. In a study of Brownian motion, \cite{einstein1956investigations} demonstrated that the mean square displacement is directly proportional to the elapsed time, expressed as $\bar{r^2} \propto t$. This prompts us to investigate a potential relationship between $(n!)^2$ and $\bar{r^2}$. Furthermore, \cite{zhao2023information} identified that the information dimension of the PMF associated with maximum RPS entropy is $2$, which aligns with the fractal dimension of Brownian motion. Collectively, these findings suggest a possible link between RPST and Brownian motion, or random walk theory.

\subsection{Random walk}
Random walk is a fundamental topic in probability theory. It is a type of stochastic process, which is a sequence of random variables that evolve over time. Random walk is formed by the successive summation of independent and identically distributed (i.i.d.) random variables. \citep{lawler2010random}. 

\subsubsection{General random walk}

\begin{definition}\label{def:random_walk}
\textnormal{(General random walk).} For $\forall t \in \mathbb{N}^{+}=\left\{ 1,2,3\dots \right\}$, let $S_{t} \in \mathbb{R}^{d}$, given a proper probability distribution $P: \mathbb{R}^d \to (0,1]$ and a group of i.i.d. random variables $\left\{ X_{t} | X_{t}\in \mathbb{R}^d , t \in \mathbb{N}^{+}\right\}$, the general random walk $S_{n}$ with step size distribution $P$ can be considered as the time-homogeneous Markov chain, defined by a summation of $\left\{ X_{t} \right\}$ :
\begin{equation}
    S_{n} = S_{0} + X_{1} +  X_{2} + \dots +X_{n},
\end{equation}
where $S_{0}\in \mathbb{R}^{d}$ is the starting point.
\end{definition}
One well-studied variant is the Gaussian random walk, which has a step size distribution of normal distribution $N(0, \sigma^2)$.

\subsubsection{Wiener process}
The Wiener process, also known as Brownian motion, is a fundamental concept in probability theory and stochastic processes. It represents the limiting behavior of a one-dimensional random walk as the step size approaches zero and the number of steps approaches infinity.
\begin{definition}
    \textnormal{(Wiener process).} For $\forall n \in \mathbb{Z}_{+}$, $W(t)$ is a Wiener process if 

\begin{align}
   W(t) =  \lim_{ n \to \infty } W_n(t) = \frac{1}{\sqrt{n}} \sum_{1 \leq i \leq \lfloor nt \rfloor} \xi_i, t \in [0,1], \xi \sim N(0,1).
\end{align}\label{eq:limit_wiener}
\end{definition}
A Wiener process $\{W(t), t > 0 \}$ has the following properties:
\begin{itemize}
    \item $W(0)$ = 0;
    \item for $0 \leq s <t$, the increments $W(t)-W(s) \sim N(0, t-s)$;
    \item For any non-overlapping interval $[s_i,t_i]$, the group of random variables $W(t)-W(s)$ are independent of each other;
    \item $W(t)$ is almost surely continuous in $t$.
\end{itemize}
Those properties will be analysed in our proposed stochastic process.

\section{Explore random walk in Random Permutation Set}\label{sec:explore random walk}

In this section, we will give an in-depth exploration of random walk in Random permutation set theory. Now let us review the motivation discussed earlier: the RPST brings order information of events to expand evidence theory, and the order information can be viewed as time sequence information since time has fixed order and flows in one direction. And the term "random" inspired us to find the connection between RPST and random walks in stochastic process.

When generating a random walk based on RPST, it is important to consider the order information present in RPST. For convenience, we use list $\left[a_1, a_2, \dots, a_n \right]$ to express order information, and simulate two-dimensional random walk. Firstly, the random variable should be defined.

One situation where the ordered information is important is the matrix multiplication, because matrix multiplication does not hold commutative property, i.e. $AB = BA$ for most of the matrix $A, B$. This inspired us to use matrix to generate a random variable. 

Given a permutation sequence, $PerS_{n \times 1}=\left(a_1, a_2, \dots, a_n \right)^{\top}$ and a arbitrary vector $\vec{V}_{0}=(x,y)^{\top}$, we want to output a random variable vector $\vec{V}_{i}=(V_{x}, V_{y})^{\top}$top. This can be done by the following computation. First we randomly generated some inversible matrices $M_{N} = \left( M_{1}, M_{2},\dots,M_{i},\dots M_{n} \right)^{\top}$, then we compute $Vec_{i}=a_{i} \cdot M_{i}\vec{V}_{0}$ for each $M_{i}$ and each $a_{i}$, getting $n$ component vectors $Vec_{i}$. Then we have a summation vector $\vec{V}_{i} = \sum_{i}Vec_{i}$ by adding all component vectors.

For the convenience of illustration, we use two-dimensional rotation matrix $R(\theta)_{2 \times 2}$ to replace $M_{i}$ in the following way: 

\begin{align} 
R(\theta)_{2 \times 2} &  = \begin{pmatrix}
\cos \theta  & \sin \theta  \\
\sin \theta  & \cos \theta
\end{pmatrix}, \\
M_{N} & =\left( R(\theta_{0}), R^{2}(\theta_{0}), \dots , R^i(\theta_{0}), \dots,R^n(\theta_{0})\right)^{\top}.  \\
\end{align}

We the use the following algorithm to generate a random variable.

\subsection{Generating random variables}\label{sec: RVG}

\begin{definition}\label{def:RVG}
(Random Variable Generator, RVG). \textnormal{Given a positive integer $n$, the random variable generator (RVG) is defined by \Cref{alg:RVG}.}
\end{definition}

\begin{algorithm}
\caption{Random Variable Generator}\label{alg:RVG}
\LinesNumbered
\SetKwFunction{RandomChoice}{RandomChoice}
\KwResult{A vector in perpendicular coordinates $(V_x, V_y)$ representing the addition of component vectors.}
\KwIn{Integer $n$ indicating the length of the set.}
\KwOut{Vector in perpendicular coordinates $(V_x, V_y)$.}
$M \gets [0]_{(n!, n)}$\\
\tcc{Initialize a zero matrix $M$ with dimensions $(n!, n)$.}
$S \gets \{1, 2, \dots, n\}$\\
\tcc{Initialize the set $S = \{1, 2, \dots, n\}$}
\For{each permutation list $p_i$ in all permutations of $S$.}{
    $M_{i} \gets p_i$ \\
}
$p_s \gets \RandomChoice (M)$  \\
\tcc{Select a possible permutation sequence $p_s$ from matrix $M$ evenly based on uniform distribution.}
$V_x, V_y \gets  0$ \\
\tcc{Initialize sum of x, y components vectors.}
\For{$i \gets  1$ \KwTo $n$}{
$\theta_i \gets  \frac{2 \pi}{n} \cdot i$\\
$x_i \gets  a_i \cdot \cos(\theta_i)$\\
$y_i \gets  a_i \cdot \sin(\theta_i)$\\
$V_x \gets  V_x + x_i$\\
$V_y \gets  V_y + y_i$\\
}
\Return{$(V_x, V_y)$}\\
\end{algorithm}

\Cref{alg:RVG} takes an integer $n$ as input, outputting a vector in perpendicular coordinates marked as a random variable. The $n$ denotes the number of component vectors, and the cardinality of a possible permutation sequence $p_i=\left(a_1, a_2, \dots, a_n \right)$. The reason of choosing possible permutation sequence will be discussed in \cref{sec:simulating random variables}. After the possible permutation sequence $p_i$ is selected, the numbers in it indicate the length of each component vector. As for the direction of each component vector, we choose to divide $2\pi$ into $n$ piece evenly, so each component vector $\vec{v_j}$ can be defined as $\vec{v_j}=(a_j, 2\cdot j \cdot \pi/n)$. Then we output the random variable $(V_x, V_y) = \vec{V_i}$ by adding all component vectors.

As shown in \cref{fig:random_variable_simulation}, we simulate $20,000$ random variables with $n$ ranging from $n=1$ to $n=12$. The index at the top of each sub-figure denotes as the number of possible random variable. (When $n \geq 7$, the number may be inaccurate due to limited simulation.), while the color in each node represents the frequency in simulation. And all numerical results are rounded to eight decimal places. 

For $n=1,2,3$, there are $1!, 2!, 3!$ kinds of values of random variables, respectively. While for $n=4$, there are not $4!=24$ but $16$ different values of random variable, as shown in the figure. This can be predicted, because when $n=4$, each component vector has a fixed direction, which are $\pi/2$, $\pi$, $3\pi/2$, and $2\pi$ respectively. This means each of these vectors points either horizontally or vertically. So each sum in the resulting vector's $x$ or $y$ direction can be produced in four ways: $[1(2-1, 3-2, 4-3), 2(3-1, 4-2), 3(4-1)]$, which yields four different combinations: $(1,3), (3,1), (2,2), (1,1)$. And since each of four ways implies a rotation direction, which in turn leads to the x and y coordinates of the vector being multiplied by either $+1$ or $-1$. Thus, there are $4 \times 4 = 16$ unique possibilities. 

This explanation can be extended to the cases of $n=5$ and $n=6$. However, the number of possible random variables grows rapidly when $n \geq 7$, compared with the simulation of $n=6$. This is intuitive due to the rapid growth of factorial.

We examined the specifics of each random variable simulation concerning $V_x$ and $V_y$. The histogram in \cref{fig:hist of value} displays the distribution of $V_x,V_y$ values in $20,000$ simulations.The x-axis and y-axis in each sub-figure represent the value and frequency, respectively. The symmetrical distribution of frequencies in each interval suggests that the expected values of $V_x, V_y$ should be zero, a finding supported by the results in \cref{fig:2-Mean_variance_plot}. It is also anticipated that as the number of simulations, denoted by $n$ tends towards infinity, $V_x, V_y$ will converge to a normal distribution.

Another important statistic property is the variance, \cref{fig:2-Mean_variance_plot} shows the variance of $V_X, V_y$ in different value of $n$. As $n$ increases, the variance of both $V_x$ and $V_y$ will grow like binomial function, which means $Var(V_{x,y}) \propto (n^2 + n)$. This variance growing speed property is another necessary feature of random walk. In \cref{fig:2-Mean_variance_plot}, we compared the variance of both $V_x$ and $V_y$, the linearity between them indicates that $V_x$ and $V_y$ are independent and symmetrical, ensuring this simulation method is like Wiener process, which is invariant to rotations.

\begin{figure}
	\centering 
	\includegraphics[width=0.45\textwidth]{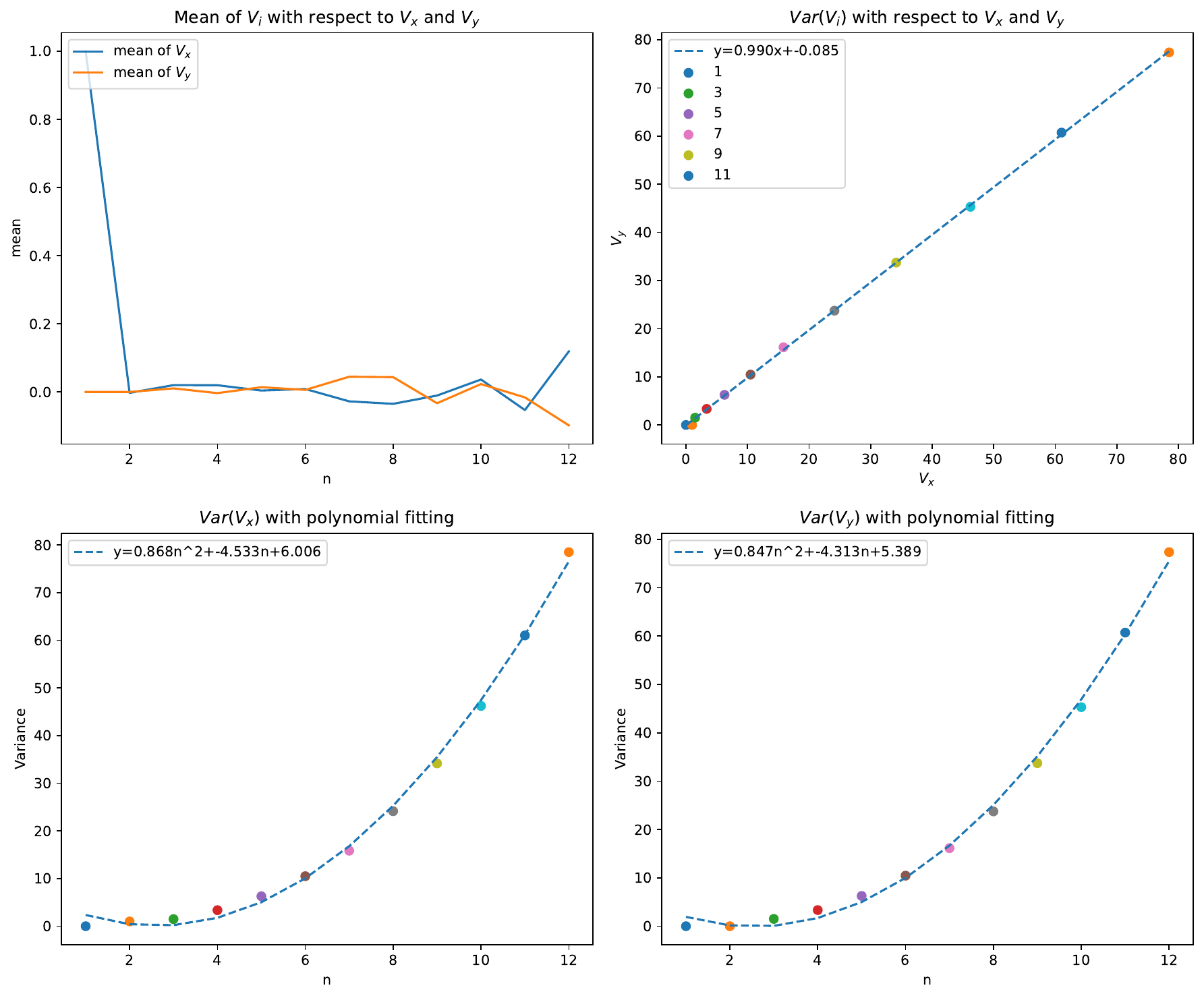}	
	\caption{Variance and mean value of $V_i$ with respect to $V_x$ and $V_y$.} 
	\label{fig:2-Mean_variance_plot}%
\end{figure}

\begin{figure*}
	\centering 
    \includegraphics[width=0.9\textwidth]{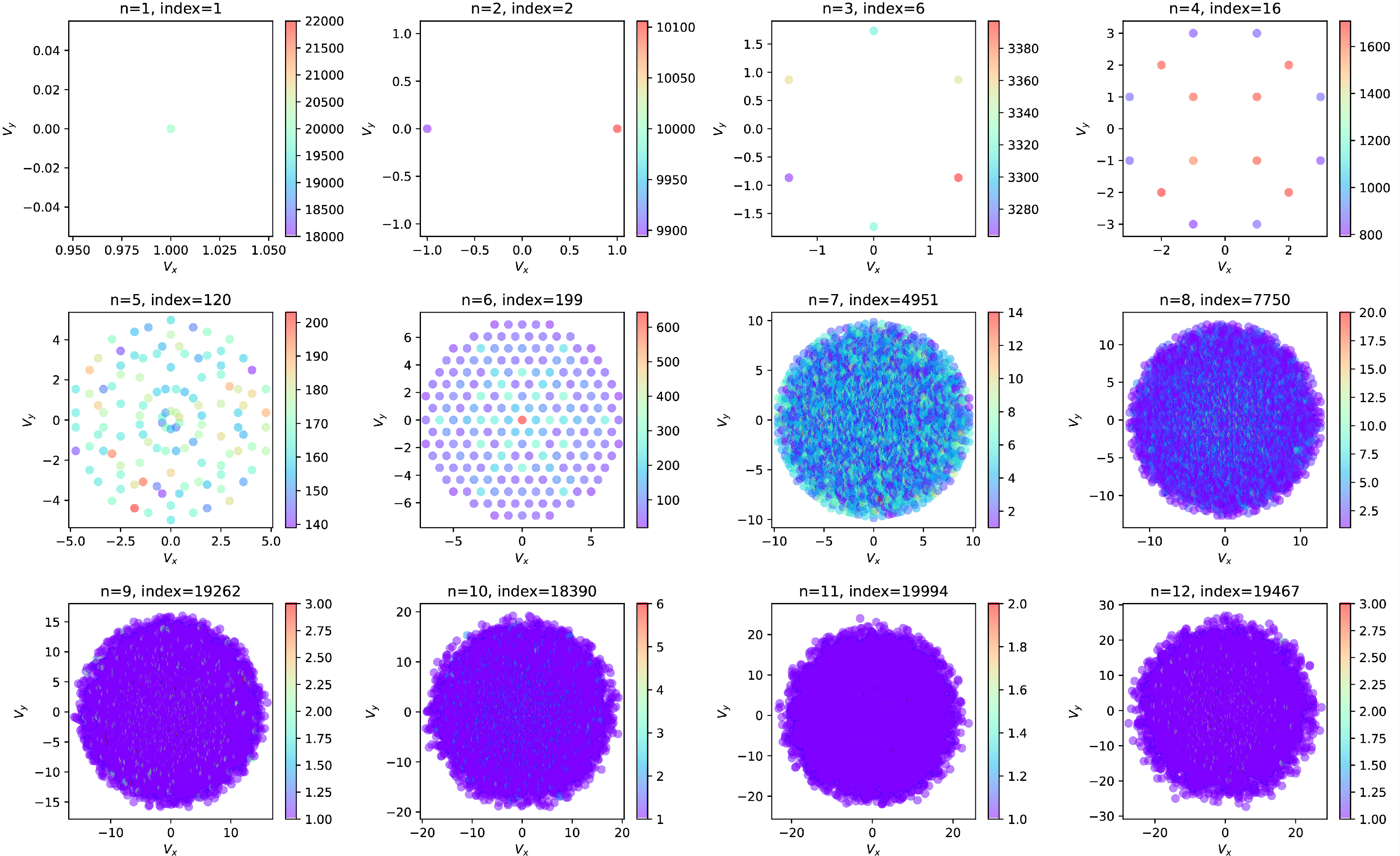}
	\caption{Visualization of random variables generation with n ranging from 1 to 12 in 20,000 simulations. Each point on the graph represents a possible random variable, with the color indicating the frequency of occurrence.} 
 \label{fig:random_variable_simulation}
\end{figure*}

\begin{figure*}
	\centering 
    \includegraphics[width=0.9\textwidth]{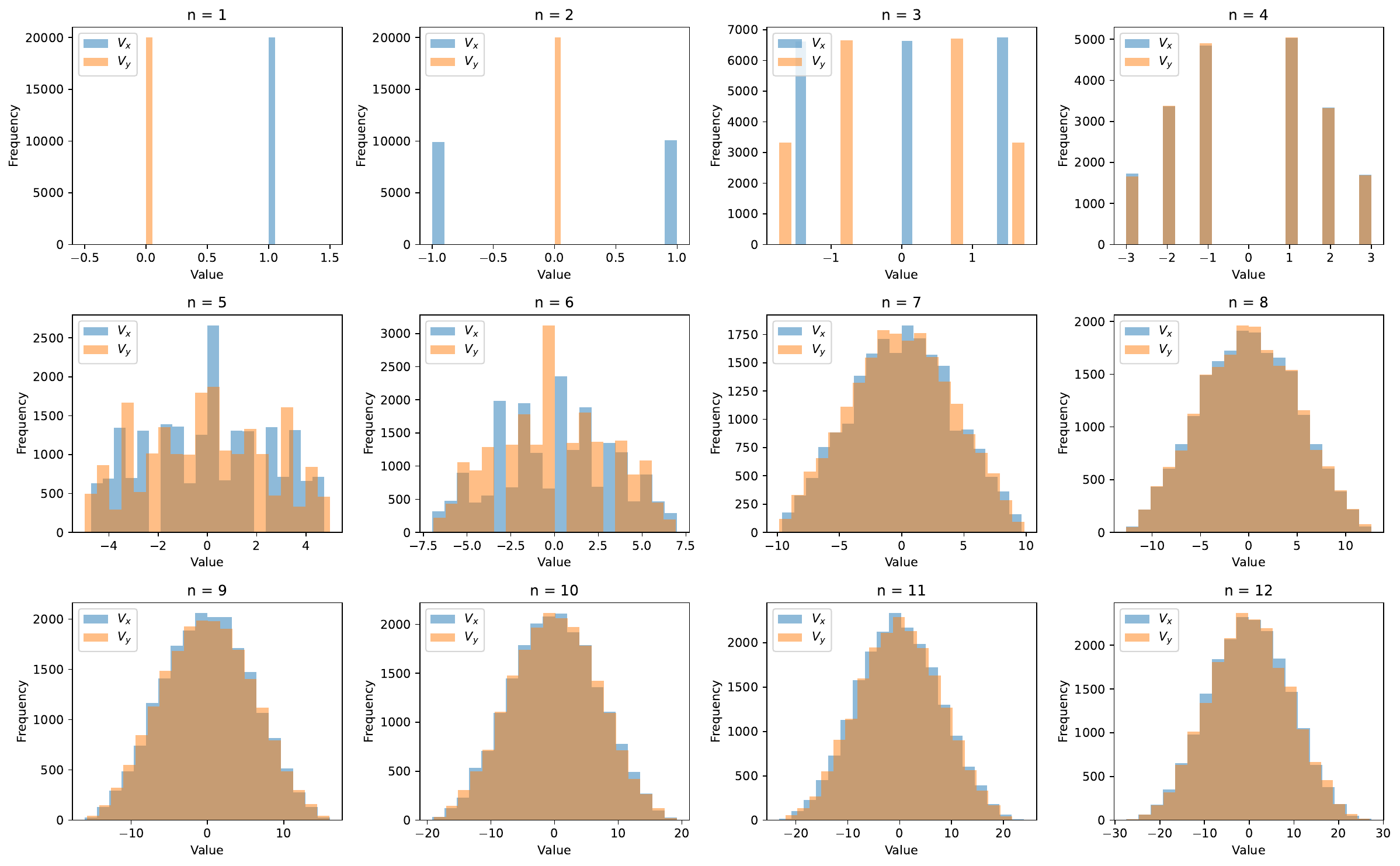}
	\caption{Histogram of random variables with respect to $V_x$ and $V_y$.} 
 \label{fig:hist of value}
\end{figure*}

\subsection{Simulating random variables for random walk}\label{sec:simulating random variables}

One common way to simulate random walk is adding a sequence of i.i.d. random variables. For example $\vec{V_i}$ form a normal distribution $N(\mu, \sigma^2)$, where $\mu$ and $\sigma$ are the mean and standard deviations of the normal distribution, respectively. Then the sum of normally distributed random variables is a random walk \citep{lawler2010random}.

\begin{align}
    S_t = \sum_{i=0}^{t} \vec{V_i},
\end{align}\label{eq:random walk generation}

where $\vec{V_i}$ is marked as a step, $V_0$ is the starting value of the random walk, and $t$ is the number of steps.
Inspired by such method, we tend to use such method with random variables simulated from RPST to generate random walk.

To generate a ideal random variable $\vec{V_i}$, we first should determine the length of the possible permutation sequence, i.e. $n$ in RVG.

Based on the maximum RPS entropy, an natural idea is to delve a distribution from RPST, and use it as probabilities associated with each possible permutation sequence. Given a fixed set $\Lambda = \left\{\lambda_1, \lambda_2, \dots, \lambda_n \right\}$, the belief assigned to possible permutation sequence whose cardinality is identical is the same. When the length of a possible permutation sequence is determined, then we can select one of the possible permutation sequence evenly as our probabilities association method, and that's why we use uniform distribution as the probabilities associated with each possible permutation sequence in \cref{sec: RVG}.

\begin{definition}\label{def:RPS_distribution}
    
\textnormal{(RPST distribution).} Given a maximum length of permutation sequences $N$, there are $P(N,n)$ choices to select a possible permutation sequences with length of $n$, then the possibility of selecting $n$ as the length of possible permutation sequence combined with the maximum RPS entropy, is defined as RPST distribution.

\begin{equation}\label{eq:RPS_distribtuion}
    P_{RPS}(n|N) = P(N,n) \cdot \mathcal{M}_{i=n,j} = \frac{P(N,n)[F(n)-1]}{\sum_{i=1}^{N}\left[ P(N,i)(F(i)-1) \right] }.
\end{equation}

\end{definition}

To illustrate the the validity of the proposed method, we consider the following way to select a possible length $n$ for permutation sequence with the same probability:

\begin{definition}\label{def:Permu_distribution}
    
\textnormal{(Permutation distribution).} Given a maximum length of permutation sequences $N$, there are $\sum_{i=1}^{N} i!$ kinds of permutation sequences, the permutation distribution is defined to choose a possible length $n$ based on the number of permutation cases.

\begin{equation}
    P_{Per}(n|N) = \frac{P(N,n)}{\sum_{i=1}^{N}P(N,n)}=\frac{P(N,n)}{F(N)-1}=\frac{P(N,n)}{\lfloor e \cdot N! \rfloor},
\end{equation}\label{eq:Permu_distribution}
\end{definition}

i.e., the possibility of selecting a possible sequence length $n$ is in proportion to the magnitude of permutation $P(N,n)$. When the  In other words, given a maximum length of permutation sequences $N$, the probability of selecting a possible permutation sequence from all $\lfloor e \cdot n! \rfloor$ sequences is $1/(\lfloor e \cdot n! \rfloor)$.

\begin{table*}
\centering
\begin{tabular}{l c c c c c c c c} 
\toprule
  &  $N$  & $n=N-5$  & $n=N-4$ & $n=N-3$ & $n=N-2$ & $n=N-1$ & $n=N$ & $\sum_{n=N-5}^{N}P(n|N)$  \\
\hline
\multirow{7}{*}{$P_{Per}(n|N)$} & 6 & 3.0700e-3 & 1.5340e-2 & 6.1350e-2 & 1.8405e-1 & 3.6810e-1 & 3.6810e-1 & 1.0000e-0 \\
 & 10 & 3.0700e-3 & 1.5330e-2 & 6.1310e-2 & 1.8394e-1 & 3.6788e-1 & 3.6788e-1 & 9.9941e-1 \\
 & $\vdots$ \\
 & 18 & 3.0700e-3 & 1.5330e-2 & 6.1310e-2 & 1.8394e-1 & 3.6788e-1 & 3.6788e-1 & 9.9941e-1 \\
 & $\vdots$ \\ 
 & $\infty$  & $\frac{1}{5!e}$ & $\frac{1}{4!e}$  & $\frac{1}{3!e}$  & $\frac{1}{2!e}$  & $\frac{1}{e}$  & $\frac{1}{e}$  &  $\frac{163}{60e}$ \\
 \addlinespace
  \hline
\multirow{7}{*}{$P_{RPS}(n|N)$} & 6  & 0.0000e-0 & 7.0000e-5 & 1.0800e-3 & 1.3820e-2 & 1.4035e-1 & 8.4468e-1 & 1.0000e-0 \\
 & 10 & 0.0000e-0 & 1.0000e-5 & 2.1000e-4 & 5.0200e-3 & 9.0430e-2 & 9.0433e-1  & 1.0000e-0  \\
 & $\vdots$ \\
 & 18 & 0.0000e-0 & 0.0000e-0 & 3.0000e-5 & 1.5500e-3 & 5.2550e-2 & 9.4587e-1 & 1.0000e-0 \\
 & $\vdots$ \\ 
 & $\infty$  & 0 & 0 & 0 & 0 & 0 & 1 & 1 \\
\bottomrule[.1em]
\end{tabular}
\caption{The last $6$ elements' probability assignment of distribution $P_{Per}(n|N)$ and $P_{RPS}(n|N)$, all results are rounded to $5$ digits.
}
\label{tab:probability distribution}
\end{table*}

\begin{figure}
	\centering 
	\includegraphics[width=0.9\textwidth]{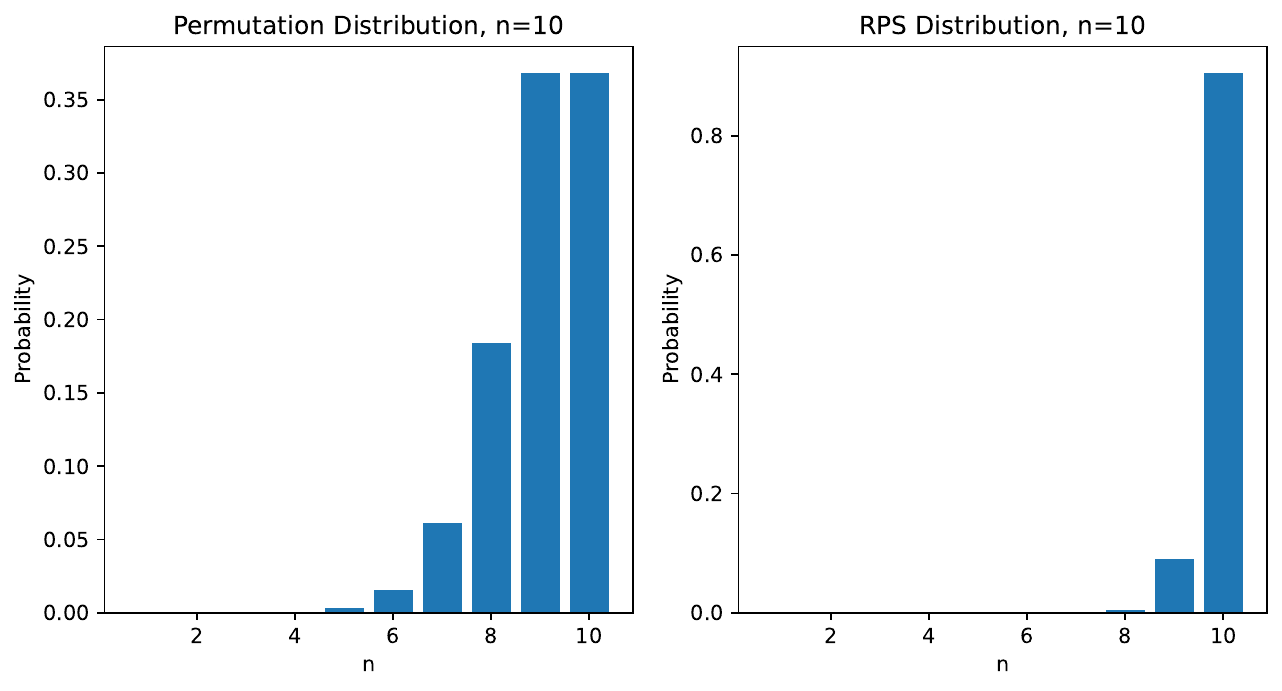}	
	\caption{Discrete probability distribution of $P_{Per}(n|N)$ and $P_{RPS}(n|N)$ with $N=10$.} 
	\label{fig:probability distribution}%
\end{figure}

We plot the discrete probability distribution of $P_{Per}(n|N)$ and $P_{RPS}(n|N)$ with $N=10$ in \cref{fig:probability distribution}. And \cref{tab:probability distribution} lists the details of the last $6$ elements' probability assignment for those two distribution. Based on above, it is obvious that the last $6$ elements take up most of the probability assignment. Thus, when selecting a possible length for permutation sequences, $P_{Per}(n|N)$ tends to choice $n$ from $[N-5, N]$, while $P_{Per}(n|N)$ like to assign most of the probability to $n=N$ with a bigger $N$. The limit form of $P_{Per}(n|N)$ and $P_{RPS}(n|N)$ will be discussed in \cref{sec: proof}.

\subsection{Generating random walk with random variables}

Using \cref{eq:random walk generation} as a construction of generating random walk, we design the following algorithm to generate random walk with random variables.

\begin{definition}\label{def:RWG}
(Random Walk Generator, RWG). \textnormal{Given a positive integer $T$ denoted as time steps, the maximum length of permutation sequence $N$, and the distribution method $P(n|N)$, the random walk generator (RWG) is defined by \Cref{alg:RWG}.}
\end{definition}

\begin{algorithm}
\caption{Random Walk Generator}\label{alg:RWG}
\LinesNumbered
\SetKwFunction{RVG}{RVG}
\SetKwFunction{GenLen}{GenLen}
\KwResult{A matrix $T_{t \times 2}$ representing the discrete time stochastic process.}
\KwIn{An integer $t$ indicating the number of time steps, a selection method $P$ ($P_{Per}(n|N)$ and $P_{RPS}(n|N)$), the maximum sequence length $N$.}
\KwOut{A matrix $T$ with size of $t \times 2$.}
$Len_Set \gets \GenLen (P, N)$\\
\tcc{Create a set of lengths with constraints to maximum length of $N$, the set is used for generating random variables.}
\For{$i = 2$ \KwTo $t$}{
    $Temp\_Vec = \RVG (Len\_Set_{i-1})$
    $T_{i} \gets T_{i-1} + Temp\_Vec$ \\
}
\tcc{Generate value of random walk at each time step.}
\Return{$T$}\\
\end{algorithm}

\begin{figure*}
	\centering 
    \includegraphics[width=0.9\textwidth]{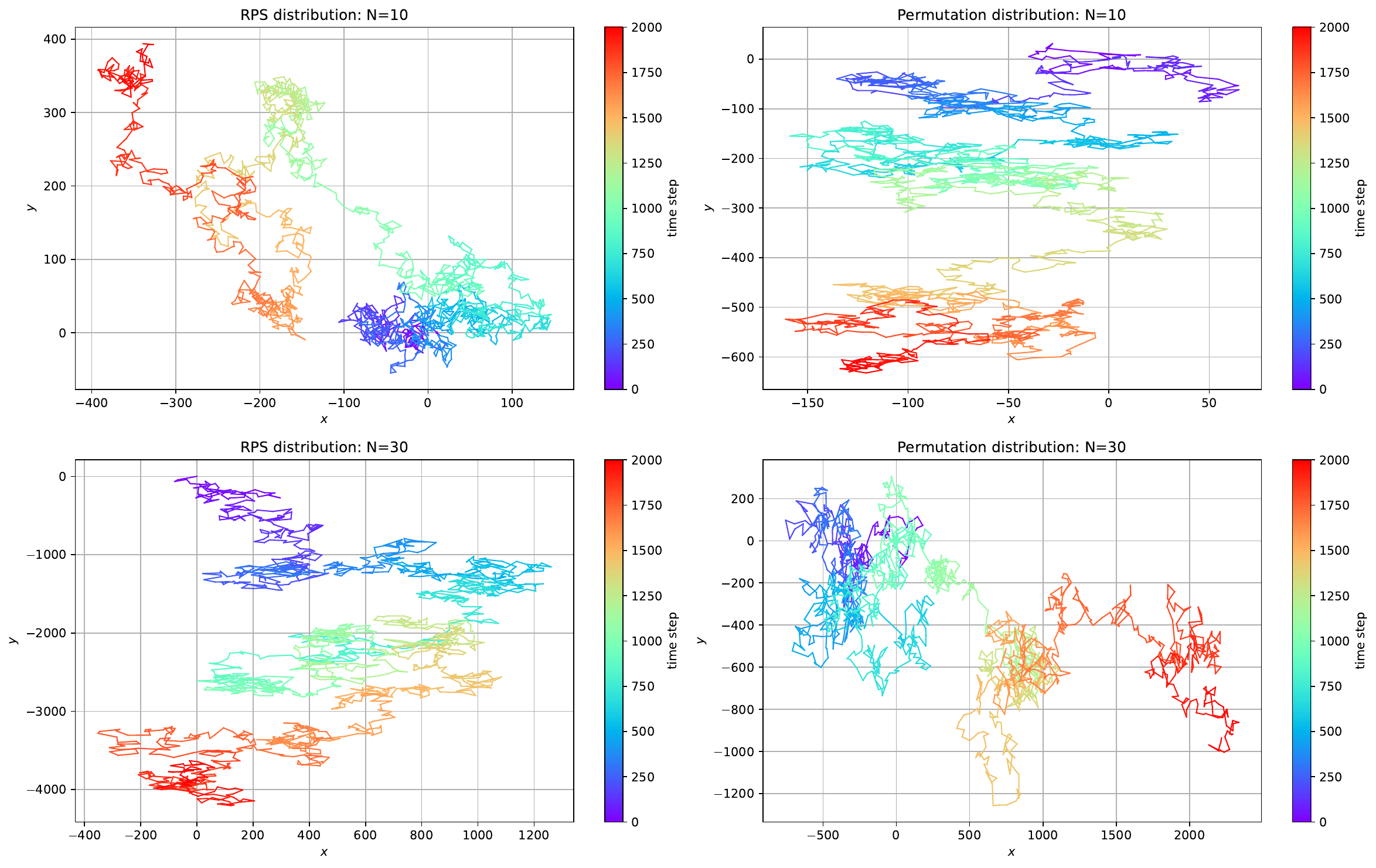}
	\caption{Visualization of random walk from distribution $P_{Per}(n|N)$ and $P_{RPS}(n|N)$, where color map is showing the time steps.} 
 \label{fig:random walk visualization}
\end{figure*}

\Cref{alg:RWG} takes the number of time steps $t$, a distribution $P$ used for selection method, and $N$ indicating the maximum length of permutation sequence, as inputs, returning a matrix $T$ storing values at each time step. 

\Cref{fig:random walk visualization} shows results across different $N$ and $P$. The color map illustrates the temporal evolution of the random walk's trajectory. As $N$ increases, the discrete-time stochastic process $T$, which is generated from RPST distribution $P_{RPS}$, shows a motion pattern resembling random walk, characterized by randomly distributed points in space.

Comparing the results of $N=10$ and $N=30$, the motion exhibits stochastic self-similarity as in random walk. This is because at each time step, this RPST-generated motion will walk through the space for each $n$ directions ($n$ being the possible length of permutation sequence with maximum length of $N$). These $n$ paths can be decomposed into $x$ and $y$ directions in perpendicular coordinates, similar to the two-dimensional random walk where the walker randomly chooses one of two perpendicular directions with a fixed step size.

To compare the proposed method's limit scale form with the Wiener process, we employ a method similar to \cref{eq:limit_wiener}, to simulate the limit scale form of the RPST-generated random walk.

\begin{align}
    RW_{n,N}(t) = \frac{\sqrt{\varrho}}{N\sqrt{N}} \frac{1}{\sqrt{n}} \sum_{1 \leq i \leq \lfloor nt \rfloor} \vec{V_i} , t \in [0,1], 
\end{align}
where $N$ is the maximum length of a permutation sequence, $n$ is the number of time step, and $\varrho$ is a variance control factor that scales the variance of RPST-generated random walk . As $N,n \to infty$, $RW_{n,N}(t)$ toward to a Wiener process, the details will be discussed in \cref{sec: proof}.

The only difference to Wiener process as a limit scale form of random walk, is that the re-scaling factor $\sqrt{\varrho}/({N\sqrt{N})}$. this is due to the fact that the random variables generated from RPST have variance growing like binomial function, as shown in \cref{fig:2-Mean_variance_plot}. So this redesigned re-scaling factor ensures the variance of random variables is invariant to $n$.

Since simulations on computers are actually discrete, and for convenience of illustration, we generate the proposed stochastic process $T$ with time step of $2,000$, and re-scale to $RW_{n,N}(t)$ with setting $\varrho = 24$.

\begin{figure*}
	\centering 
    \includegraphics[width=0.9\textwidth]{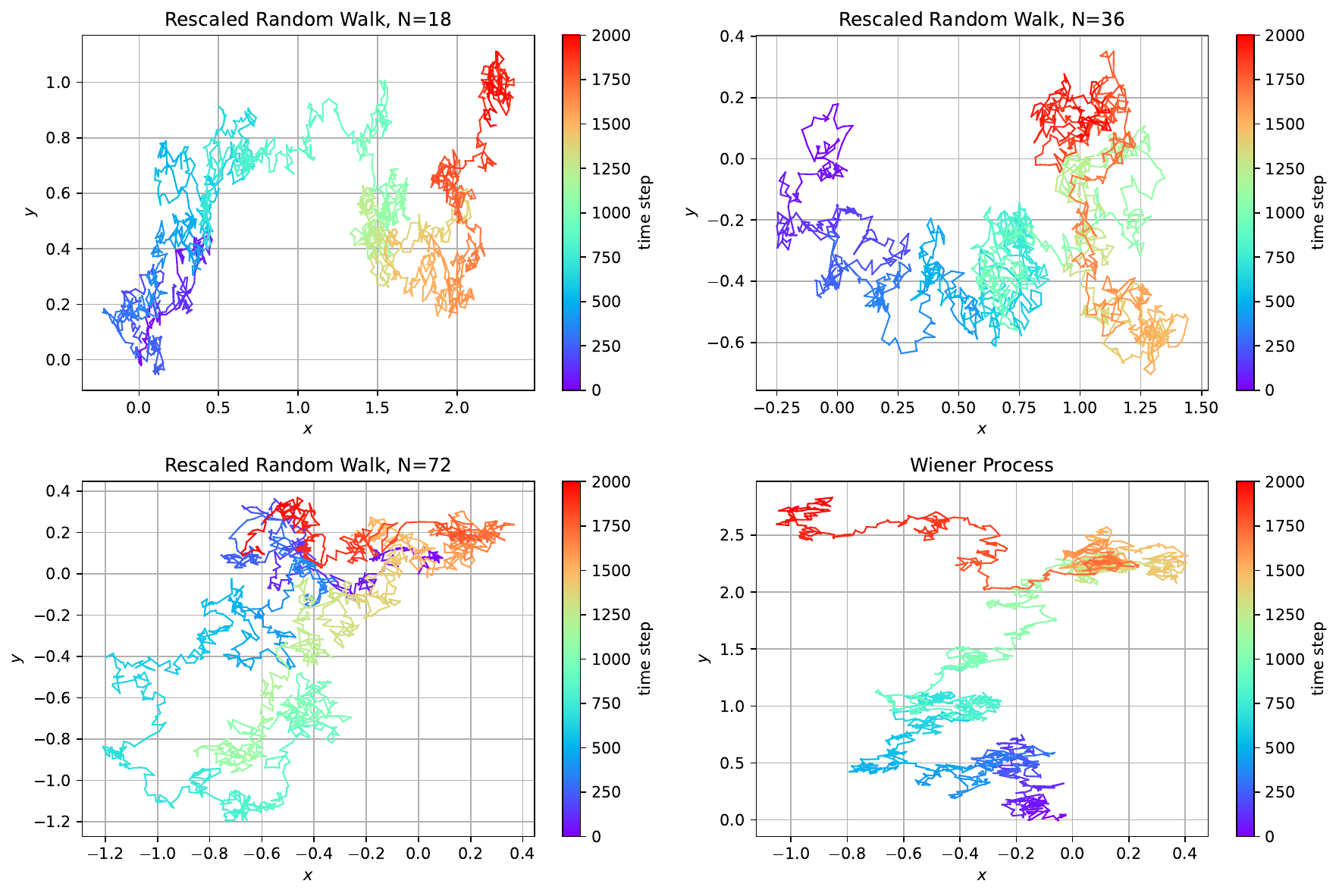}
	\caption{Scaled Random walk with different $N$ and Wiener process with time steps $2,000$ and variance control factor $\varrho=24$.} 
 \label{fig:rescaled random walk with wiener process}
\end{figure*}

\begin{figure}[ht]
	\centering 
	\includegraphics[width=0.9\textwidth]{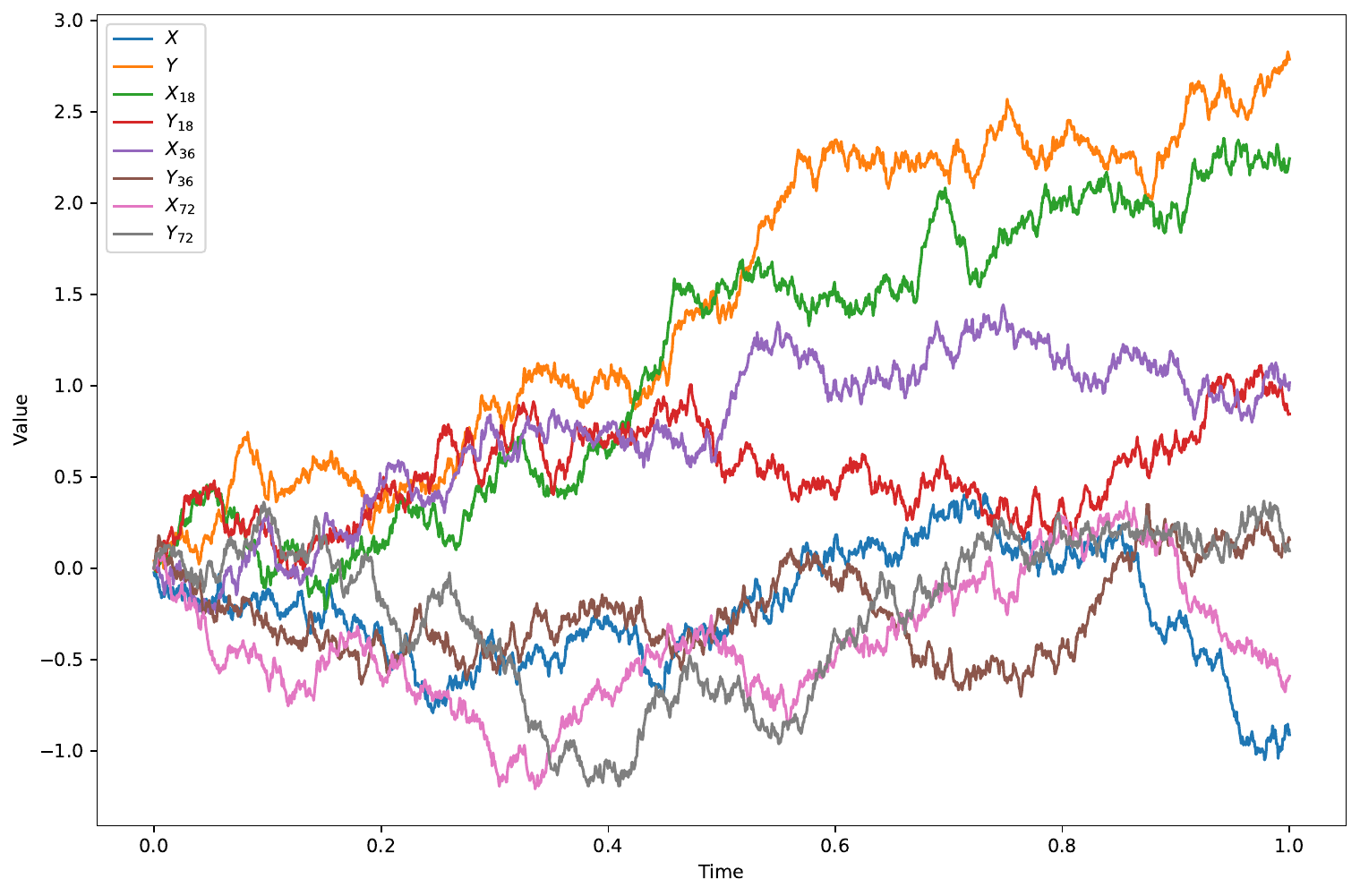}	
	\caption{Component value of random walk and Wiener process of \cref{fig:rescaled random walk with wiener process}, the $2,000$ time steps are converted to time $t\in [0,1]$ for convenience and $\varrho$ is set to 24 $\varrho=24$.} 
	\label{fig:component_plot}%
\end{figure}

As shown in \cref{fig:rescaled random walk with wiener process}, the scaling RPST-generated random walk  do visually seem the same as standard Wiener process, not only the randomly walked point path, but also the boundaries. And more details about the component values about $X$ and $Y$ axis are plotted in \cref{fig:component_plot}. Based on this result, it seems that the proposed stochastic process converges to Wiener process as $N$ increases. However, additional verification is required before reaching a definitive conclusion.

In \cref{component_variance_re}, the mean and variance values of various stochastic processes are compared to the Wiener process across different $\varrho$ values, with time steps and sample processes limited to $100$ and $200$, respectively. And the $5$ interval are set to $[0,20)$, $[0,40)$, $[0,60)$, $[0,80)$, $[0,100)$ to minimize errors. 

Results show that all proposed methods with different $N$ exhibit properties similar to the Wiener process in terms of mean value and variance, where the mean value is zero and variance scales with time steps. Compared the sub-figures in \cref{component_variance_re}, the difference lies in the slope of variance, which is why we introduce the variance control factor $\varrho$ to regulate the variance of the proposed stochastic process. This ensures that the variance of the process aligns with that of a Wiener process.

From the results and analysis presented above, it is evident that as the sample size ($N$) increases, the RPST-generated random walk converges to a Wiener process, which is the limit scale form of a two-dimensional random walk. This demonstrates the successful derivation of a random walk from RPST.

\begin{figure}[ht]
	\centering 
	\includegraphics[width=0.9\textwidth]{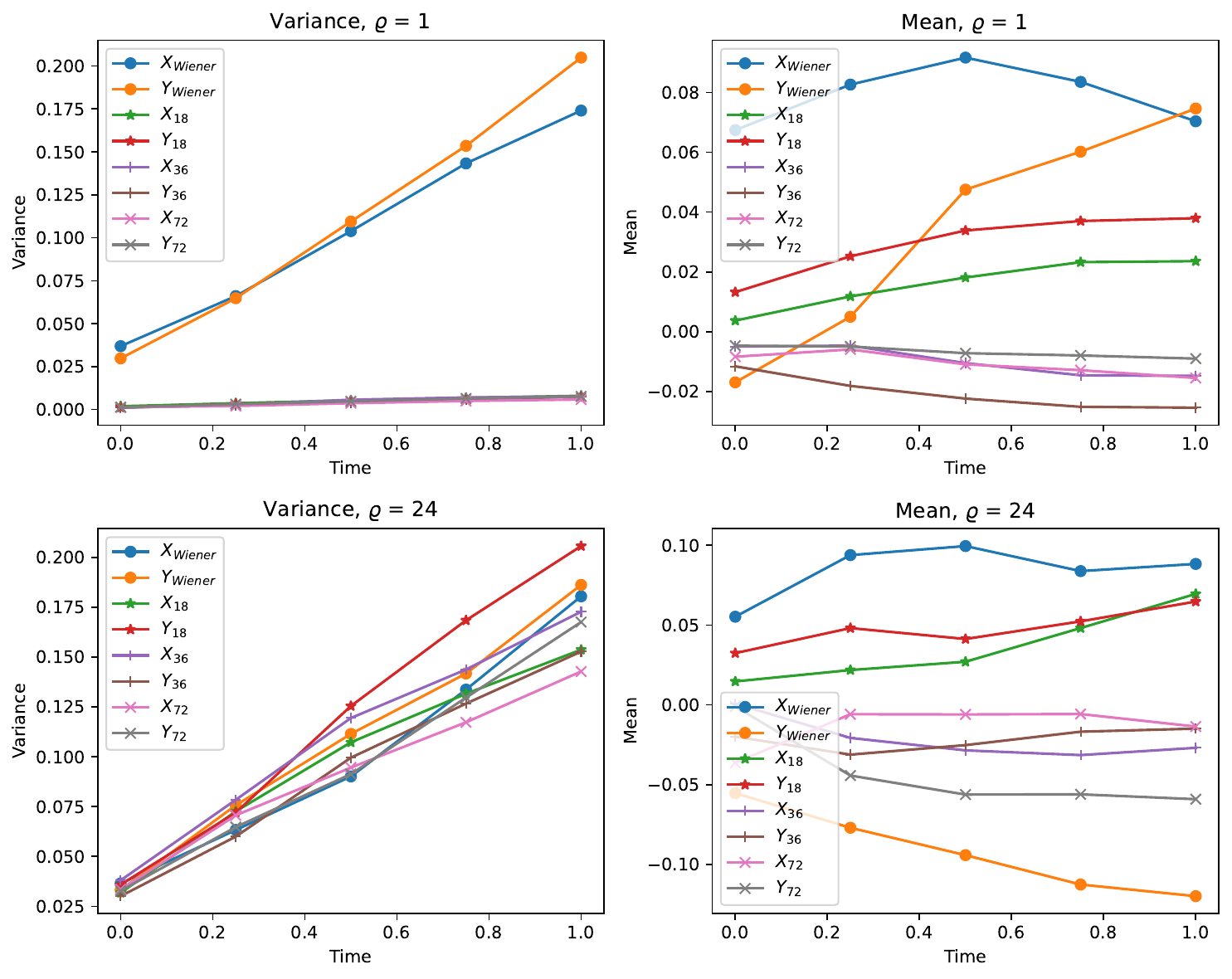}
	\caption{Variance (left) and mean (right) value of Wiener process $W(t)$ and limit scale form of random walk from RPST $RW_{n,N}(t)$ across various $N$ and variance control factor, time steps and number of simulations are set to $100, 200$, respectively. } 
	\label{component_variance_re}%
\end{figure}

\section{Conclusion}\label{sec:conclusion}

Random permutation set theory (RPST) is a promising extension of evidence theory that introduces ordered information to its reasoning framework. The indexed order in RPST can be viewed as a time series, which motivates the exploration of a connection between RPST and random walk, a fundamental topic in probability theory. This paper demonstrates that RPST can be used to construct a Gaussian random walk and, in the limit, a Wiener process. The established link between RPST and random walk provides insights into the physical meaning of RPST and enables its application in existing random walk domains. This not only expands the application scope of RPST but also provides insights for combination the strengths of both RPST and random walk for problem-solving.

Future investigations should concentrate on overcoming the limitations of current study. This may involve elucidating the physical implications of RPST through its association with random walks. Subsequently, the application of this random walk model to real-world scenarios, such as epidemiological modeling, financial market analysis and machine learning algorithm, could be explored.

\section*{ACKNOWLEDGMENTS}
The work is partially supported by National Natural Science Foundation of China (Grant No. 62373078).

\section*{AUTHOR DECLARATIONS}

\subsection*{Conflict of Interest}
The authors have no conflicts to disclose.

\subsection*{CRediT authorship contribution statement}
\textbf{Jiefeng Zhou}: Conceptualization, Methodology, Formal analysis, Investigation, Writing-original draft, Writing-review \& editing. \textbf{Zhen Li}: Validation. \textbf{Yong Deng}: Writing-review \& editing, Supervision, Project administration, Funding acquisition.

\subsection*{Data Availability}
The data that support findings of this study are available from the corresponding author upon reasonable request.




\appendix

\section{Proof of deriving random walk from RPST}\label{sec: proof}
In this section, we will analyze the RPST-generated random walk in detail and demonstrate its similarities with random walk in mathematics.

\subsection{Analysis on random variables}
In \cref{sec:explore random walk}, the random variables are first defined for generating random walk. these variables are generated using the order property of RPST. As a simulation method, its important statistic properties like expected value and variance should be reviewed.

\subsubsection{Expected value analysis}

\begin{lemma}\label{lem:expected value}
\textnormal{(Expected value of a random variable).} The expected value of a random variable generated with RVG is zero, namely,
\begin{align}\label{eq:expected value of RVG}
    \mathbb{E}\left[ V_{i} \right] =0.
\end{align}
\end{lemma}

\begin{proof}
As described in \cref{alg:RVG}, when dealing with an integer set of length $N$, the likelihood of selecting a specific permutation sequence is equal, with a probability of 
\begin{align}
    P \left\{ V=V_{i} \right\}= \frac{1}{N!}.
\end{align}
To determine the expected value in each direction, we calculate the frequency of numbers appearing in a fixed direction, such as $\frac{2\pi \cdot i}{N}$. The magnitude of this direction in a simulation is determined by 
\begin{align}
   |V_{component}| = P \left\{ V=V_{i} \right\} \cdot \sum_{j=1}^{N}\left[j\cdot\left( N-1 \right)!\right]=\frac{1+N}{2}.
\end{align}\label{eq:mag of variable}
Due to symmetry in direction generation and identical magnitudes, the resultant sum vector $\left( V_{x}, V_{y}  \right)$ is anticipated to yield a value of $0$. Thus, the expected value of $\left( V_{x}, V_{y} \right)$ or $V_{i}$ is 
\begin{align}
    \mathbb{E} \left[ \left( V_{x}, V_{y} \right) \right] =\mathbb{E}\left[ V_{i} \right]=0
\end{align}\label{eq:expected_value_random_variable}
\end{proof}

\cref{fig:2-Mean_variance_plot} also displays the mean value of $V_x, V_y$ in $20,000$ simulations, suggesting the expected value of $\vec{V_i}$ is zero.

\subsubsection{Variance analysis}
Variance is a measure of dispersion, which is pretty useful in generating random walk. As shown in \cref{fig:2-Mean_variance_plot}, the variance of $V_x$ and $V_y$ are quantitatively identical, and both of them exhibit a binomial growth rate with respect to $N$. 

\begin{lemma}\label{lem:variance of random vector}
\textnormal{(Variance of a random variable).} The variance of a random variable generated from RVG is a binomial function on $N$, namely:
\begin{align}
    Var(\vec{V_i}) \propto (N^2+N).
\end{align}
\end{lemma}

\begin{proof}
In \cref{fig:2-Mean_variance_plot}, it can be directly observed that the variance of both $V_x$ and $V_y$ is in proportion to $N^2$. And this relationship can be explained from the following perspectives.

As indicated in \cref{eq:mag of variable}, the expected value of magnitude in each direction is proportional to the maximum length of permutation sequence $N$, while the expected value of $\vec{V_i}$ remains zero, as demonstrated in \cref{lem:expected value}. This ensures that as $N$ increases, the distribution of random variables maintains its symmetry, resembling a round boundary. The size of this boundary is determined by the value of $N$, as shown in \cref{fig:random_variable_simulation}. Therefore, there exists  a critical value $N_0$, such that when $N_2 > N_1 \geq N_0$, random variables $\vec{V_i}_2$ generated with $N=N_2$ can be represented by the random variables $\vec{V_i}_1$ generated with $N=N_1$, denoted as

\begin{align}
    \vec{V_i}_2 = f(N_2-N_1) \cdot \vec{V_i}_1,
\end{align}\label{eq:linear scale vector}

 where $f(x)$ is a function $\mathbb{R} \to \mathbb{R}$, as shown in \cref{eq:mag of variable}.

Then based on the propagation property of variance:
\begin{align}
    Var(a V) = a^2 Var(V),
\end{align}\label{eq:propa_var}
where $a$ is a constant, and \cref{eq:linear scale vector}, one can easily delve such result in \cref{lem:variance of random vector}.
\end{proof}

\subsection{Analysis on permutation distribution and RPST distribution}

\begin{lemma}
    \textnormal{(Limit form of RPST distribution). } When $N \to infty$, the RPST distribution will converge to the following form:

\begin{align}
    \lim_{ N \to \infty } P_{RPS}(n=N | N) =1
\end{align}
\end{lemma}

\begin{proof}
    
The RPST distribution is based on the maximum RPS entropy, this distribution will surely converge to $P(n=N|N)=1$, as suggested in \cref{tab:probability distribution}. This result is determined by its definition on \cref{eq:RPS_distribtuion}. In our previous work \citep{zhou2024limit}, we proved that 

\begin{align}\label{eq:limit of maximum RPS}
\lim_{ N \to \infty } \sum_{i=1}^{N} \left[ P(N,i)\left( F(i)-1 \right)  \right] - e \cdot (N!)^{2} = 0,\\
F(N)-1 = \lfloor e \cdot N! \rfloor -1
\end{align}
Compared with $P(N,n)[F(n)-1]$, we get 

\begin{align}
\lim_{ N \to \infty } P_{RPS}(n=N | N) & = \frac{N! \left( \lfloor e \cdot N! \rfloor-1 \right) }{e(N!)^{2}}  \nonumber\\
& = \lim_{ N \to \infty } \frac{N! \lfloor e \cdot N! \rfloor }{e(N!)^{2}} -  \lim_{ N \to \infty } \frac{1}{e(N!)^{2}} \nonumber\\
 & = 1- 0 \nonumber\\
 & =1
\end{align}

This result ensures that when $N$ is bigger enough, this distribution will converge to the following probability distribution:
\begin{align}\label{eq:limit_RPS_distribution}
P_{RPS}(n | N) =
\begin{cases}
1,  & n=N ; \\
0,  & others.
\end{cases}
\end{align}

\end{proof}

This probability distribution can be explained by the maximum entropy principle. This principle states that the distribution with the highest entropy is the most likely to represent the current state of a system. Therefore, the larger the value of $N$, the more likely it is that the system will choose the permutation sequence with the maximum length. This is because a longer permutation sequence indicates more uncertainty.

However, the probability assignment in Permutation distribution will not converge to a single element. Conducted from \cref{eq:Permu_distribution}, we have

\begin{align}
\lim_{ N \to \infty } P_{Per}(n |N)  & = \lim_{ N \to \infty }\frac{P(N,n)}{\lfloor e \cdot N! \rfloor} \nonumber \\
 &  =\lim_{ N \to \infty } \frac{N!}{(N-n)!}/\left( \lfloor e \cdot N! \rfloor \right)   \nonumber\\
 & = \lim_{ N \to \infty } \frac{1}{e(N-n)!}.
\end{align}

Thus, permutation distribution will converge to the limit form as shown in \cref{tab:probability distribution}. In contrast to the RPST distribution, the permutation distribution exhibits a greater degree of variability, which hinders the generation of i.i.d. random variables. Consequently, it is not suitable for simulating random walks.

\subsection{Analysis on RPST-generated random walk}
In previous section, it is proved that RPST distribution will converge to a probability distribution shown in \cref{eq:limit_RPS_distribution}, ensuring its generation of i.i.d. random variables, which is a necessity for generating random walk. And some statistics properties of RPST-generated random walk are analysed in this section.

The histogram in \cref{fig:hist of value} displays the generation of random variables. It is expected that, for a fixed value of $N$, the RVG will produce random vectors that adhere to a normal distribution. This convergence towards a normal distribution is controlled by the central limit theorem (CLT) and Donsker's theorem, which ensure that as $N \to \infty$, the summation vector $\Vec{V_i} = (V_x, V_y)$ will be distributed according to a normal distribution $N(0, \sigma^2_N)$, where $\sigma^2_N=f(N)\sigma^2$ and $f(N)\propto N^2$ is a binomial function with respect to $N$. This result is supported by \cref{lem:variance of random vector} and \cref{fig:2-Mean_variance_plot}.

Due to the binomial growth of variance, we design the re-scaling factor $\sqrt{\varrho} / (N \sqrt{N})$ to fitting the variance pattern of Wiener process. This re-scaling factor comes from the following theorem:

\begin{theorem}\label{the:wiener_process}
    The RPST-generated random walk $RW(t)$ can be converted to Wiener process if the following limit form exists:
\begin{align}
    W(t) = \lim_{n,N \to \infty} RW_{n,N}(t) = \lim_{n,N \to \infty} \frac{\sqrt{\varrho}}{N\sqrt{N}} \frac{1}{\sqrt{n}} \sum_{1 \leq i \leq \lfloor nt \rfloor} \vec{V_i} , t \in [0,1].
\end{align}

\end{theorem}

\begin{proof}
The difference between $RW(t)$ and Gaussian random walk lies in the variance, if we can scale the variance of $RW(t)$ to fit into normal distribution, then following \cref{eq:limit_wiener} one can easily proves it.

As demonstrated in \cref{lem:expected value} and \cref{lem:variance of random vector}, we get
\begin{align}
\mathbb{E}\left[ RW(t) \right] = \mathbb{E}\left[ \sum_{i=0}^{t}V_{i} \right] =0.
\end{align}\label{eq:expected_value_RW(t)}

Then we have 

\begin{align}
\mathbb{E}\left[ RW^{2}(t) \right]  = \sum_{i=1}^{t}\mathbb{E}\left[ V_{i}^{2} \right]  + 2 \sum_{1\leq i < j <t}\mathbb{E} \left[  V_{i} V_{j} \right] .
\end{align}\label{eq:ERWtsr}

Since random variables are independent with each other, then for any $i \neq j$, $\mathbb{E} \left[  V_{i} V_{j} \right]=0$. 

Using the equation of variance 

\begin{align}
Var(X) = \mathbb{E}\left[ X^{2} \right] - \mathbb{E}\left[ X \right] ^{2},
\end{align}

we get

\begin{align}
\mathbb{E}\left[ V_{i}^{2} \right]  = Var(V_{i}) \propto (N^{2} + N).
\end{align}

Together with \cref{eq:ERWtsr}, we get 

\begin{align}
\mathbb{E}\left[ RW^{2}(t) \right] =\sum_{i=1}^{t}\mathbb{E}\left[ V_{i}^{2} \right] = t \cdot \left( Var(V_{i})  \right) \propto t(N^{2} +N).
\end{align}

Finally, the variance of $RW(t)$ has the following property

\begin{align}
Var(RW(t)) = \mathbb{E} \left[ RW^{2}(t) \right]  \propto 
t(N^{2} + N).
\end{align}

Compared with Gaussian random walk, the growing speed of $Var(RW(t))$ is additionally multiplied by $N^{2} +N$. In other words, after $t-s$ steps, the increments of Gaussian random walk $Z_{t-s} \sim N(0, (t-s)\sigma^{2})$, while in random walk from RPST we have $Z_{t-s} \sim N(0, (t-s)(N^{2}+N)\sigma^{2})$. Thus, based on the propagation property of variance, we construct $V_i'=V_i/{N\sqrt{N}}$ to offset the term $(N^{2}+N)$. But there still exists a coefficient between the scaling $V_i$ and the step of Gaussian random walk (or Wiener process in limiting form), as shown in \cref{component_variance_re}, so we design a variance control factor $\varrho$, then the final scaling factor would be $\sqrt{\varrho}/({N\sqrt{N}})$. We set $\varrho=24$ to approximating Gaussian random walk.

 So after scaling the RPST-generated random walk to Gaussian random walk, one can use Donsker's theorem and \cref{eq:limit_wiener} to construct the form in \cref{the:wiener_process} to convert a RPST-generated random walk to a Wiener process, thus \cref{the:wiener_process} is proved.

\end{proof}

Similar to Wiener process, the limit scale form of random walk from RPST $RW_{n,N}(0)$ is also characterised by the following properties:

\begin{itemize}
    \item $RW_{n,N}(0)= 0$. This prosperity is achieved by setting the starting point to zero $V_0 = 0$.
    \item $RW_{n,N}(t)$ has independent increments. This property is determined by the fact that each random variables are independent with each other, following a step size distribution of normal distribution $N(0, f(N)\sigma^2)$.
    \item For any $0 \leq s <t$, the increments $RW_{n,N}(t) - RW_{n,N}(s) \sim N(0, t-s)$. This can be done by setting the variance factor $\varrho=24$, as shown in \cref{component_variance_re}.
    \item $RW_{n,N}(t)$ is almost surely continuous in $t$, this is ensured by Donsker's theorem as $N,n \to \infty$.
\end{itemize}

\bibliography{aipsamp}

\end{document}